\documentclass{article}

    \usepackage[final]{neurips_2021}

\usepackage{hyperref}       
\usepackage{url}            
\usepackage{amsfonts}       
\usepackage{nicefrac}       
\usepackage{xcolor}         
\usepackage{amsmath}
\usepackage{amsfonts}
\usepackage{bbm}
\usepackage{mathtools}
\usepackage{amssymb}
\usepackage{amsmath}
\usepackage{enumerate}
\usepackage{comment}
\usepackage[textwidth=8em,textsize=small]{todonotes}

\DeclarePairedDelimiter\ceil{\lceil}{\rceil}
\DeclarePairedDelimiter\floor{\lfloor}{\rfloor}
\usepackage{amsthm}
\newtheorem{remark}{Remark}
\newtheorem{theorem}{Theorem}[section]
\newtheorem{lemma}[theorem]{Lemma}
\newtheorem{assump}{Assumption}
\newtheorem{cond}{Condition}

\newtheorem{defi}[theorem]{Definition}

\numberwithin{equation}{section}

\newtheorem{innercustomgeneric}{\customgenericname}
\providecommand{\customgenericname}{}
\newcommand{\newcustomtheorem}[2]{%
  \newenvironment{#1}[1]
  {%
   \renewcommand\customgenericname{#2}%
   \renewcommand\theinnercustomgeneric{##1}%
   \innercustomgeneric
  }
  {\endinnercustomgeneric}
}

\newcustomtheorem{customthm}{Theorem}
\newcustomtheorem{customlemma}{Lemma}

\def\hf{\hat{f}}
\def\cF{\mathcal{F}}
\def\hg{\hat{g}}
\def\he{\hat{e}}
\def\cG{\mathcal{G}}
\def\cE{\mathcal{E}}
\def\cH{\mathcal{H}}
\def\Rbb{\mathbb{R}}
\def\Ebb{\mathbb{E}}
\def\cX{\mathcal{X}}
\def\hbnu{\hat{\boldsymbol{\nu}}}
\def\hbmu{\hat{\boldsymbol{\mu}}}

\title{Non-Asymptotic Error Bounds for \\ Bidirectional GANs}

\author{%
Shiao Liu\\
  Department of Statistics and Actuarial Science,
  University of Iowa\\
  Iowa City, IA 52242, USA \\
  \texttt{shiao-liu@uiowa.edu}
  \And
Yunfei Yang \thanks{Corresponding authors} \\
Department of Mathematics,
The Hong Kong University of Science and Technology\\
 Clear Water Bay,
 Hong Kong, China\\
 \texttt{yyangdc@connect.ust.hk}\\
\And Jian Huang {$^*$}\\
  Department of Statistics and Actuarial Science,
  University of Iowa\\
  Iowa City, IA 52242, USA \\
  \texttt{jian-huang@uiowa.edu} \\
\And
  Yuling Jiao{$^*$}\\
  School of Mathematics and Statistics,
				Wuhan University\\
Wuhan, Hubei, China 430072\\
\texttt{yulingjiaomath@whu.edu.cn} \\
  \And
Yang Wang \\
Department of Mathematics,
The Hong Kong University of Science and Technology\\
 Clear Water Bay,
 Hong Kong, China\\
 \texttt{yangwang@ust.hk}\\
}

\begin{document}

\maketitle

\begin{abstract}
We derive nearly sharp bounds for the bidirectional GAN (BiGAN) estimation error under the Dudley distance between the latent joint distribution and the data joint distribution with appropriately specified  architecture of the neural networks used in the model. To the best of our knowledge, this is the first theoretical guarantee for the bidirectional GAN learning approach. An appealing feature of our results is that they do not assume the reference and the data distributions to have the same dimensions or these distributions to have bounded support. These assumptions are commonly assumed in the existing convergence analysis of the unidirectional GANs but may not be satisfied in practice. Our results are also applicable to the Wasserstein bidirectional GAN if the target distribution is assumed to have a bounded support. To prove these results, we construct neural network functions that push forward an empirical distribution to another arbitrary empirical distribution on a possibly different-dimensional space. We also develop a novel decomposition of the integral probability metric for the error analysis of bidirectional GANs. These basic theoretical results are of independent interest and can be applied to other related learning problems.
\end{abstract}

\section{Introduction}\label{sec1}

Generative adversarial networks (GAN) \citep{goodfellow} is an important approach to implicitly learning and sampling from high-dimensional complex distributions.
GANs have been shown to achieve impressive performance in many machine learning tasks \citep{radford2016unsupervised, reed16, zhu17,karras2018progressive,karras2019stylebased, brock2019large}.
%
Several recent studies have generalized GANs 
to bidirectional generative learning, which learns an encoder mapping the data distribution to the reference distribution simultaneously together with the generator doing reversely.
These studies include the adversarial autoencoder (AAE) \citep{makhzani2015adversarial}, bidirectional GAN  (BiGAN) \citep{donahue2016adversarial}, adversarially learned inference (ALI) \citep{dumoulin2016adversarially}, and bidirectional generative modeling using adversarial gradient estimation (AGES) \citep{shen2020bidirectional}.
A common feature of these methods is that they generalize the basic adversarial training framework of the original GAN from unidirectional to bidirectional. \cite{dumoulin2016adversarially} showed that BiGANs make use of the joint distribution of data and latent representations, which can better capture the information of data than the vanilla GANs.
Comparing with the unidirectional GANs, the joint distribution matching in the training of bidirectional GANs alleviates mode dropping and encourages cycle consistency  \citep{shen2020bidirectional}.

Several elegant and stimulating papers have analyzed the theoretical properties of unidirectional GANs.
\citet{arora} considered the generalization error of GANs under the neural net distance.
\citet{zhang2018on} improved the generalization error bound in \citet{arora}.
\citet{liang2020} studied the minimax optimal rates for learning distributions with empirical samples under Sobolev evaluation class and density class. The minimax rate is $O(n^{-{1}/{2}}\vee n^{-{\alpha+\beta}/{(2\alpha+\beta)}})$, where $\alpha$ and $\beta$ are the regularity parameters for Sobolev density and evaluation class, respectively.
 \citet{bai2018} analyzed the estimation
error of GANs under the Wasserstein distance for a special class of distributions implemented by
a generator, while the discriminator is designed to guarantee zero bias.
 \citet{zhao} studies the convergence properties of GANs when both the evaluation class and the target density class are H\"older classes and derived $O(n^{-{\beta}/{(2\beta+d)}}\log^2 n)$ bound, where $d$ is the dimension of the data distribution and  $\alpha$ and $\beta$ are the regularity parameters for H\"older density and evaluation class, respectively.
While impressive progresses have been made on the theoretical understanding of GANs,
there are still some drawbacks in the existing results. For example,
\begin{enumerate}[(a)]
\setlength\itemsep{-0.03 cm}
\item The reference distribution and the target data distribution are assumed to have the same dimension,  which is not the actual setting for GAN training.
\item  The reference and the target data distributions are assumed to be supported on bounded sets.

\item The prefactors in the convergence rates may
depend on the dimension $d$ of the data distribution exponentially.
\end{enumerate}
In practice, GANs are usually trained using a reference distributions with a lower dimension  than that of the target data distribution. Indeed, an important strength of GANs is that they can model low-dimensional latent structures via using a low-dimensional reference distribution.
The bounded support assumption excludes some commonly used Gaussian  distributions as the reference.
Therefore, strictly speaking, the existing  convergence analysis results do not apply to what have been done in practice. In addition,  there has been no theoretical analysis of  bidirectional  GANs
in the literature.

\subsection{Contributions} 
We derive nearly sharp non-asymptotic bounds for the GAN estimation error under the Dudley distance between the reference joint distribution and the data joint distribution.
To the best of our knowledge, this is the first result providing theoretical guarantees for bidirectional GAN estimation error rate.
We do not assume that the reference and the target data distributions have the same dimension
or these distributions have bounded support.
 Also, our results are applicable to the Wasserstein distance if the target data distribution is assumed to have a bounded support.

The main novel aspects of our work are as follows.
\begin{enumerate}[(1)]
\setlength\itemsep{-0.05 cm}
\item We allow the dimension of the reference distribution  to be different from the dimension of the target distribution, in particular, it can be much lower than that of the target distribution.


 \item We allow unbounded support for the reference distribution and the target distribution under mild  conditions on the tail probabilities of the target distribution.

 \item We explicitly establish that the prefactors in the error bounds depend on  the square root of the dimension of the target distribution. This is a significant improvement over the
     exponential dependence on $d$ in the existing works. 



\end{enumerate}

{\color{black}
Moreover,  we develop a novel decomposition of integral probability metric for the error analysis of bidirectional GANs.
We also show that the pushforward distribution of an empirical distribution based on neural networks can perfectly approximate another arbitrary empirical distribution as long as the number of discrete points are the same.
}

{\color{black}
\textbf{Notation } We use $\sigma$ to denote the ReLU activation function in neural networks, which is $\sigma(x)=\max\{x,0\}, x \in \mathbb{R}$. We use $I$ to denote the identity map. Without further indication, $\|\cdot\|$ represents the $L_2$ norm. For any function $g$, let $\|g\|_{\infty}=\sup_{x}\|g(x)\|$.
We use notation $O(\cdot)$ and $\tilde{O}(\cdot)$ to express the order of function slightly differently, where $O(\cdot)$ omits the universal constant independent of $d$ while $\tilde{O}(\cdot)$ omits the constant depending on $d$. We use $B_2^{d}(a)$ to denote the $L_2$ ball in $\mathbb{R}^d$ with center at $\mathbf{0}$ and radius $a$.
Let $g_{\#}\nu$ be the pushforward distribution of $\nu$ by function $g$ in the sense that $g_{\#}\nu(A)=\nu(g^{-1}(A))$ for any measurable set $A$. We use $\hat{\Ebb}$ to denote taking expectation with respect to the empirical distribution.
}

\section{Bidirectional generative learning}
\label{sect2}

We describe the setup of the bidirectional GAN estimation problem and present the assumptions we need in our analysis.

\subsection{Bidirectional GAN estimators}\label{sec22}

Let $\mu$ be the target data distribution supported on $\mathbb{R}^d$ for $d \ge 1.$
Let $\nu$ be a reference distribution which is easy to sample from.
We first consider the case when $\nu$ is supported on $\mathbb{R}$,
and then extend it to
$\mathbb{R}^k$, where $k \ge 1$ can be different from $d$. Usually, $k \ll d$ in
practical machine learning tasks such as image generation.
The goal is to learn functions $g: \Rbb \to \Rbb^d$ and $e: \Rbb^d \to \Rbb$ such that $\tilde{g}_{\# }\nu =\tilde{e}_{\# }\mu$,
where $\tilde{g}:=(g,I)$ and $\tilde{e}:=(I,e)$, $\tilde{g}_{\# }\nu$ is the pushforward distribution of $\nu$ under $\tilde{g}$ and $\tilde{e}_{\# }\mu$ is the pushforward distribution of $\mu$ under $\tilde{e}$. We call $\tilde{g}_{\# }\nu$ the joint latent distribution or joint reference distribution and $\tilde{e}_{\# }\mu$ the joint data distribution or joint target distribution.
At the population level,  the bidirectional GAN solves the minimax problem:
\[
(g^*, e^*, f^*) \in \arg\min _{g\in \mathcal{G}, e\in \mathcal{E}}\max_{f\in\mathcal{F}} {\mathbb{E}}_{Z\sim\nu}[f(g(Z),Z)]-{\mathbb{E}}_{x\sim\mu}[f(X, e(X))],
\]
where $\mathcal{G}, \mathcal{E}, \mathcal{F}$ are referred to as the generator class, the encoder class, and the discriminator class, respectively.
Suppose we have two independent random samples $Z_1, \ldots, Z_n \stackrel{i.i.d.}{\sim}\nu$ and $X_1, \ldots, X_n\stackrel{i.i.d.}{\sim}\mu$.
At the sample level, the bidirectional GAN solves the empirical version of the above minimax problem:
\begin{align}
    (\hg_{\theta} , \he_{\varphi} , \hf_{\omega})
    &
    = \arg\min _{g_{\theta}\in \mathcal{G}_{NN}, e_{\varphi}\in\mathcal{E}_{NN}}\max_{f_{\omega}\in\mathcal{F}_{NN}} \frac{1}{n}\sum_{i=1}^{n}f_{\omega}(g_{\theta}(Z_i), Z_i)
    -\frac{1}{n}\sum_{j=1}^{n}f_{\omega}(X_j, e_{\varphi}(X_j)) ,
     \label{eq:2.1}
\end{align}
where $\mathcal{G}_{NN}$ and $\mathcal{E}_{NN}$ are two classes of neural networks approximating the generator class $\mathcal{G}$ and the encoder class $\mathcal{E}$ respectively, and $\mathcal{F}_{NN}$ is a class of neural networks approximating the discriminator class $\mathcal{F}$.

\subsection{Assumptions}\label{sec2}
We assume the target $\mu$ and the reference $\nu$ satisfy the following assumptions.
\begin{assump}[Subexponential tail]\label{asp1}
For a large $n$, the target distribution $\mu$ on $\mathbb{R}^d$ and the reference distribtuion $\nu$ on $\Rbb$ satisfies the first moment tail condition for some $\delta>0$,
\[
   \max\{\mathbb{E}_{\nu} \|Z\| \mathbbm{1}_{\{\|Z\|>\log n\}}, \mathbb{E}_{\mu} \|X\| \mathbbm{1}_{\{\|X\|>\log n\}}\} = O(n^{-\frac{(\log n)^{\delta}}{d}}).
\]
\end{assump}

\begin{assump}[Absolute continuity] \label{asp2}
Both the target distribution $\mu$ on $\Rbb^d$ and the reference distribution $\nu$ on $\mathbb{R}$ are absolutely continuous with respect to the Lebesgue measure $\lambda.$
\end{assump}

Assumption \ref{asp1} is a technical condition for dealing with the case when  $\mu$ and $\nu$ are supported on $\mathbb{R}^d$ and $\mathbb{R}$ instead of compact subsets. For distributions with bounded supports, this assumption is automatically satisfied.
Here the factor $(\log n)^{\delta}$ ensures that the tails of $\mu$ and $\nu$ are sub-exponential, and it can be easily satisfied if the distributions are sub-gaussian.
For the reference distribution, Assumption \ref{asp1} and \ref{asp2}  can be easily satisfied by specifying $\nu$ as some common distribution with easy-to-sample density such as   Gaussian or uniform, which is usually done in the applications of GANs.
For the target distribution, Assumption \ref{asp1} and \ref{asp2} specifies the  type of distributions that are learnable by bidirectional GAN with our theoretical guarantees. Note that Assumption \ref{asp1} is also necessary in our proof for bounding the generator and encoder approximation error in the sense that the results will not hold if we replace $(\log n)^{\delta}$ with 1. Assumption \ref{asp2} is also necessary for Theorem \ref{lma6} in mapping between empirical samples, which is essential in bounding generator and encoder approximation error.

\subsection{Generator, encoder and discriminator classes}\label{sec21}
Let $\mathcal{F}_{NN}:=\mathcal{NN}(W_1,L_1)$ be the discriminator class consisting of the feedforward ReLU neural networks $f_{\omega}: \mathbb{R}^{d+1}\mapsto \mathbb{R}$ with width $W_1$ and depth $L_1$.  Similarly, let $\mathcal{G}_{NN}:=\mathcal{NN}(W_2, L_2)$ be the generator class consisting of the feedforward ReLU neural networks
$g_{\theta}: \mathbb{R}\mapsto \mathbb{R}^d $ with width $W_2$ and depth  $L_2$, and $\mathcal{E}_{NN}:=\mathcal{NN}(W_3, L_3)$ the encoder class consisting of the feedforward ReLU neural networks
$e_{\varphi}: \mathbb{R}^d\mapsto \mathbb{R} $ with width $W_3$ and depth  $L_3$.

The functions $f_{\omega}\in \mathcal{F}_{NN}$ have the following form:
\begin{align*}
    f_{\omega}(x)=A_{L_1}\cdot \sigma(A_{L_1-1}\cdots \sigma(A_1 x+b_1)\cdots +b_{L_1-1})+b_{L_1}
\end{align*}
where $A_i$ are the weight matrices with number of rows and columns no larger than the width $W_1$, $b_i$ are the bias vectors with compatible dimensions, and $\sigma$ is the ReLU activation function $\sigma(x)=x\vee 0$. Similarly, functions $g_{\theta}\in \mathcal{G}_{NN}$ and $e_{\varphi}\in \mathcal{E}_{NN}$ have the following form:
\begin{align*}
    g_{\theta}(x)&=A'_{L_2}\cdot \sigma(A'_{L_2-1}\cdots \sigma(A'_1 x+b'_1)\cdots +b'_{L_2-1})+b'_{L_2}\\
    e_{\varphi}(x)&=A''_{L_3}\cdot \sigma(A''_{L_3-1}\cdots \sigma(A''_1 x+b''_1)\cdots +b''_{L_3-1})+b''_{L_3}
\end{align*}
where $A_i'$ and $A_i''$ are the weight matrices with number of rows and columns no larger than $W_2$ and $W_3$, respectively,
and  $b_i'$ and $b_i''$ are the bias vectors with compatible dimensions.

We impose the following conditions on $\mathcal{G}_{NN}$, $\mathcal{E}_{NN}$, and $\mathcal{F}_{NN}$.
\begin{cond}\label{cond:1} For any $g_{\theta}\in\mathcal{G}_{NN}$ and $e_{\varphi}\in\mathcal{E}_{NN}$, we have
$
\max\{\|g_{\theta}\|_{\infty},\|e_{\varphi}\|_{\infty}\}\leq \log n.
$
\end{cond}
Condition \ref{cond:1} on $\mathcal{G}_{NN}$ can be easily satisfied 
by adding an additional clipping layer $\ell$ after the original output layer, with $c_{n,d}\equiv  {(\log n)}/{\sqrt{d}}$,
\begin{align}
\label{clip1}
    \ell(a)=a\wedge c_{n,d} \vee (-c_{n,d})=\sigma(a+c_{n,d})-\sigma(a-c_{n,d})-c_{n,d}.
\end{align}
We truncate the output of $\|g_{\theta}\|$ to an increasing interval $[-\log n,\log n]$ to include the whole $\mathbb{R}^d$ support for the evaluation function class. Condition \ref{cond:1} on $\mathcal{E}_{NN}$ can be satisfied in the same manner. This condition is technically necessary in our proof (see appendix).
%

\section{Non-asymptotic error bounds} 
\label{sect3}
We characterize the bidirectional GAN solutions based on minimizing
the  
integral probability metric (IPM,  \citet{muller}) between two distributions $\mu$ and $\nu$ with respect to a symmetric evaluation function class $\mathcal{F}$, defined by
\begin{align}
\label{ipmF}
    d_{\mathcal{F}}(\mu,\nu)=\sup_{f\in \mathcal{F}}[\mathbb{E}_{\mu}f-\mathbb{E}_{ \nu}f].
\end{align}
By specifying the evaluation function class $\mathcal{F}$ differently, we can 
obtain many commonly-used metrics \citep{liu2017}. Here we focus on the following two
\begin{itemize}
\setlength\itemsep{-0.03 cm}
   \item $\mathcal{F}= $ bounded Lipschitz function class $: d_{\mathcal{F}}=d_{BL}$, (bounded Lipschitz (or Dudley) metric: metrizing weak convergence,  \citet{dudley2018real}),
    \item $\mathcal{F}= $ Lipschitz function class $: d_{\mathcal{F}}=W_1$ (Wasserstein GAN, \citet{arjovsky17}).
\end{itemize}

We  consider the estimation error under the Dudley metric $d_{BL}$.
Note that in the case when $\mu$ and $\nu$  have bounded supports, the Dudley metric $d_{BL}$ is equivalent to 
the 1-Wasserstein metric $W_1$. Therefore, under the bounded support condition for $\mu$ and $\nu$,
all our convergence results also hold under the Wasserstein distance $W_1$. Even if the support of $\mu$ and $\nu$ are unbounded, we can still apply the result of \citet{lu2020universal} to avoid empirical process theory and obtain an stochastic error bound under the Wasserstein distance $W_1$.
However, the result of \citet{lu2020universal} requires sub-gaussianity to obtain the $\sqrt{d}$ prefactor. In order to make it more general, we use the empirical processes theory to get the explicit prefactor.  Also, the discriminator approximation error will be unbounded if we consider the Wasserstein distance $W_1$. Hence, we can only consider $d_{BL}$ for the unbounded support case.

The bidirectional GAN solution $(\hg_{\theta},\he_{\varphi})$ in \eqref{eq:2.1} also minimizes the distance between $(\tilde{g}_{\theta})_{\#}\hat{\nu}_n$ and $(\tilde{e}_{\varphi})_{\#}\hat{\mu}_n$ under $d_{\mathcal{F}_{NN}}$
\begin{align*}
\underset{g_{\theta}\in\mathcal{G}_{NN}, e_{\varphi}\in\mathcal{E}_{NN}}{\min}d_{\mathcal{F}_{NN}}
((\tilde{g}_{\theta})_{\#}\hat{\nu}_n,(\tilde{e}_{\varphi})_{\#}\hat{\mu}_n).
\end{align*}
 However, even if two distributions are close with respect to $d_{\mathcal{F}_{NN}}$, there is no automatic guarantee that they will still be close
 under other 
 metrics, for example,  the Dudley or the Wasserstein distance ~ \citep{arora}. Therefore, it is natural
to ask the question:
\begin{itemize}
    \item How close are the two bidirectional GAN estimators $\hbnu:=(\hg_{\theta}, I)_{\#}\nu$ and $\hbmu:=(I, \he_{\varphi})_{\#}\mu$ under some other stronger 
        metrics?
\end{itemize}
We consider the IPM with the uniformly bounded 1-Lipschitz function class on $\mathbb{R}^{d+1}$, as the evaluation class, which is defined as, for some finite $B > 0$,
\begin{align}
\mathcal{F}^1:=\big\{f: \  & \mathbb{R}^{d+1}\mapsto \mathbb{R} 
\big| \  |f(x)-f(y)|\leq \|x-y\|,
x,y\in \mathbb{R}^{d+1}
\text{ and }\|f\|_{\infty}\leq B \big\}
\label{eq:2.2}
\end{align}

In Theorem \ref{cor1}, we consider the bounded support case where $d_{\cF}=W_1$; In Theorem \ref{thm1}, we extend the result to the unbounded support case; In Theorem \ref{cor2}, we extend the result to the case where the dimension of the reference distribution is arbitrary.

We first present a result when $\mu$ is supported on a compact subset
$[-M, M]^d\subset\mathbb{R}^d$  and $\nu$ is supported on
$[-M, M]\subset\mathbb{R}$ for a finite $M > 0$.
\begin{theorem}
\label{cor1}
Suppose that the target $\mu$ is supported on $[-M, M]^d\subset \mathbb{R}^d$ and the reference $\nu$ is supported on
$[-M, M]\subset\mathbb{R}$ for a finite $M > 0$, and Assumption~\ref{asp2} holds.
{\color{black} Let the outputs of $g_{\theta}$ and $e_{\varphi}$ be within $[-M, M]^d$ and $[-M, M]$ for
$g_{\theta} \in \mathcal{G}_{NN}$ and $e_{\varphi} \in \mathcal{E}_{NN}$, respectively. 
}
 By specifying the three network structures as $W_1 L_1\ge \ceil*{\sqrt{n}}$, $W_2^2 L_2= C_1 d n $, and $W_3^2 L_3= C_2  n $ for some constants $12\leq C_1, C_2\leq 384$ and properly choosing parameters, we have
\begin{align*}
\mathbb{E}d_{\mathcal{F}^1}(\hbnu,\hbmu)\leq C_0 \sqrt{d} n^{-\frac{1}{d+1}}(\log n)^{\frac{1}{d+1}},
\end{align*}
where $C_0>0$ is a constant independent of $d$ and $n$.
\end{theorem}
The prefactor $C_0\sqrt{d}$ in the error bound depends on $d^{1/2}$ linearly. This is different from the
existing works where the dependence of the prefactor on $d$ is either not clearly described or is exponential.
In high-dimensional settings with large $d$, this makes a substantial difference in the quality of the error bounds. These remarks apply to all the results stated below.

The next theorem deals with the case of unbounded support.

\begin{theorem}
\label{thm1}
Suppose Assumption~\ref{asp1} and~\ref{asp2} hold, and Condition \ref{cond:1} is satisfied. By specifying the structures of the three network classes as $W_1 L_1\ge\ceil*{\sqrt{n}}$, $W_2^2 L_2= C_1 d n $, and $W_3^2 L_3= C_2  n $ for some constants $12\leq C_1, C_2\leq 384$ and properly choosing parameters, we have
\begin{align*}
    \mathbb{E}d_{\mathcal{F}^1}(\hbnu,\hbmu)\leq
    \min\big\{C_0 \sqrt{d} n^{-\frac{1}{d+1}}(\log n)^{1+\frac{1}{d+1}}, C_d n^{-\frac{1}{d+1}}\log n\big\},
\end{align*}
where $C_0$ is a constant independent of $d$ and $n$, but $C_d$ depends on $d$.
\end{theorem}
Note that two methods are used in bounding stochastic errors (see appendix), which leads to two different bounds: one with an explicit $\sqrt{d}$ prefactor with the cost that we have an additional $\log n$ factor. Another one with an implicit prefactor but with a better $\log n$ factor.  Hence, it is a tradeoff between the explicitness of prefactor and the order of $\log n$.




Our next result generalizes the results to the case when the reference distribution $\nu$ is supported on $\mathbb{R}^k$ for $k\in \mathbb{N}_+.$

\begin{assump}\label{asp3}
The target distribution $\mu$ on $\Rbb^d$ is absolutely continuous with respect to the
Lebesgue measure on $\Rbb^d$ and the reference distribution $\nu$ on $\mathbb{R}^k$
is absolutely continuous with respect to the Lebesgue measure on $\Rbb^k$,
and $k\ll d$.
\end{assump}

With the above assumption, we have the following theorem providing theoretical guarantees for the validity of any dimensional reference $\nu$.

\begin{theorem}\label{cor2}
Suppose Assumption~\ref{asp1} and~\ref{asp3} hold, and Condition \ref{cond:1} is satisfied. By specifying generator and discriminator class structure as $W_1 L_1\ge\ceil*{\sqrt{n}}$, $W_2^2 L_2= C_1 d n $, and $W_3^2 L_3= C_2 k n $ for some constants $12\leq C_1, C_2\leq 384$ and properly choosing parameters, we have
\begin{align*}
    \mathbb{E}d_{\mathcal{F}^1}(\hbnu,\hbmu)\leq
    \min\big\{C_0 \sqrt{d} n^{-\frac{1}{d+k}}(\log n)^{1+\frac{1}{d+k}}, C_d n^{-\frac{1}{d+k}}\log n\big\},
\end{align*}
where $C_0$ is a constant independent of $d$ and $n$, but $C_d$ depends on $d$.
\end{theorem}

Note that the errors bounds established in Theorems \ref{cor1}-\ref{cor2} are tight up to a logarithmic factor, since the minimax rate measured in Wasserstein distance for learning distributions when the Lipschitz evaluation class is defined on $\Rbb^d$ is $\tilde{O}(n^{-\frac{1}{d}})$ \citep{liang2020}.

\section{Approximation and stochastic errors}\label{sect6}
In this section we present a novel inequality for decomposing the total error into approximation and stochastic errors and establish bounds on these errors.
\subsection{Decomposition of the estimation error}\label{sec3}

Define the approximation error of a function class $\mathcal{F}$ to another function class $\mathcal{H}$ by 
$$\mathcal{E}(\mathcal{H},\mathcal{F}):=\underset{h\in\mathcal{H}}{\sup}\underset{f\in\mathcal{F}}
{\inf}\|h-f\|_{\infty}.$$

We decompose the Dudley distance $d_{\mathcal{F}^1}(\hbnu,\hbmu)$
between the latent joint distribution and the data joint distribution into four different error terms,
\begin{itemize}
\setlength\itemsep{-0.03 cm}
\item  
 the approximation error of the discriminator class $\mathcal{F}_{NN}$ to $\mathcal{F}^1$:
$$\mathcal{E}_1=\mathcal{E}(\mathcal{F}^1,\mathcal{F}_{NN}),$$
\item the approximation error of the generator and encoder classes:
    $$\mathcal{E}_2=\inf_{g_{\theta}\in\mathcal{G}_{NN},e_{\varphi}\in\mathcal{E}_{NN}}\sup_{f_{\omega} \in\mathcal{F}_{NN}}\frac{1}{n}\sum_{i=1}^{n}\Big(f_{\omega}(g_{\theta}(z_i),z_i)-f_{\omega}(x_i,e_{\varphi}(x_i)) \Big),$$
\item the stochastic error for the latent joint distribution $\hbnu$: $$\mathcal{E}_3=\sup_{f_{\omega}\in\mathcal{F}^1} \mathbb{E} f_{\omega}(\hat{g}(z),z)- \hat{\mathbb{E}} f_{\omega}(\hat{g}(z),z),$$
\item  the stochastic error for the latent joint distribution $\hbmu$:  $$\mathcal{E}_4=\sup_{f_{\omega}\in\mathcal{F}^1}\hat{\mathbb{E}} f_{\omega}(x,\hat{e}(x))
    -\mathbb{E} f_{\omega}(x,\hat{e}(x)).$$
\end{itemize}


\begin{lemma}\label{lma4} Let $(\hg_{\theta}, \he_{\varphi})$ be the bidirectional GAN solution in \eqref{eq:2.1} and $\mathcal{F}^1$ be the uniformly bounded 1-Lipschitz function class defined in \eqref{eq:2.2}. Then the Dudley distance between the latent joint distribution $\hbnu=(\hg_{\theta}, I)_{\#}\nu$ and the data joint distribution $\hbmu=(I, \he_{\varphi})_{\#}\mu$ can be decomposed as follows
\begin{eqnarray}
\label{ErrDecomp}
    d_{\mathcal{F}^1}(\hbnu,\hbmu)
      & \leq &  2\mathcal{E}_1+\mathcal{E}_2+\mathcal{E}_3+\mathcal{E}_4.
\end{eqnarray}
\end{lemma}
The novel decomposition (\ref{ErrDecomp})  is fundamental to our error analysis.
Based on  (\ref{ErrDecomp}),  we  bound each error term on the right side
of  (\ref{ErrDecomp}) and balance the bounds to obtain an overall bound for the bidirectional GAN estimation.


For proving Lemma \ref{lma4}, we introduce the following useful inequality, which states that for any two probability distributions, the difference in IPMs with two distinct evaluation classes will not exceed 2 times the approximation error between the two evaluation classes, that is,
%
for any probability distributions $\mu$ and $\nu$ and symmetric function classes $\mathcal{F}$ and $\mathcal{H}$,
\begin{align}
\label{dIPM}
    d_{\mathcal{H}}(\mu,\nu)-d_{\mathcal{F}}(\mu,\nu)\leq 2\mathcal{E}(\mathcal{H},\mathcal{F}).
\end{align}
It is easy to check that if we replace $d_{\mathcal{H}}(\mu,\nu)$ by $\hat{d}_{\mathcal{H}}(\mu,\nu):=\underset{h\in\mathcal{H}}{\sup}
[\hat{\mathbb{E}}_{\mu}h-\hat{\mathbb{E}}_{\nu}h]$, (\ref{dIPM}) still holds.

\begin{proof} [Proof of Lemma \ref{lma4}]
We have
\begin{align*}
    d_{\mathcal{F}^1}(\hbnu,\hbmu)=&\sup_{f_{\omega}\in\mathcal{F}^1} \mathbb{E} f_{\omega}(\hat{g}(z),z)-\mathbb{E} f_{\omega}(x,\hat{e}(x))\\
    \le &\sup_{f_{\omega}\in\mathcal{F}^1} \mathbb{E} f_{\omega}(\hat{g}(z),z)- \hat{\mathbb{E}} f_{\omega}(\hat{g}(z),z)+ \sup_{f_{\omega}\in\mathcal{F}^1}\hat{\mathbb{E}} f_{\omega}(\hat{g}(z),z)-\hat{\mathbb{E}} f_{\omega}(x,\hat{e}(x))\\
    &+\sup_{f_{\omega}\in\mathcal{F}^1}\hat{\mathbb{E}} f_{\omega}(x,\hat{e}(x))
    -\mathbb{E} f_{\omega}(x,\hat{e}(x))\\
    =& \mathcal{E}_3+\mathcal{E}_4+ \sup_{f_{\omega}\in\mathcal{F}^1}\hat{\mathbb{E}} f_{\omega}(\hat{g}(z),z)
    -\hat{\mathbb{E}} f_{\omega}(x,\hat{e}(x)).
\end{align*}
Denote $A:=\sup_{f_{\omega}\in\mathcal{F}^1}\hat{\mathbb{E}} f_{\omega}(\hat{g}(z),z)
    -\hat{\mathbb{E}} f_{\omega}(x,\hat{e}(x))$. By (\ref{dIPM}) 
    and the optimality of the bidirectional GAN solutions, $A$ satisfies
\begin{align*}
   A&=\sup_{f_{\omega}\in\mathcal{F}^1}\frac{1}{n}\sum_{i=1}^{n}\Big(f_{\omega}(\hat{g}(z_i),z_i)-f_{\omega}(x_i,\hat{e}(x_i)) \Big)\\
    & \leq \sup_{f_{\omega}\in\mathcal{F}_{NN}}\frac{1}{n}\sum_{i=1}^{n}\Big(f_{\omega}(\hat{g}(z_i),z_i)-f_{\omega}(x_i,\hat{e}(x_i)) \Big)+2\mathcal{E}(\mathcal{F}^1,\mathcal{F}_{NN})
    \\
    &= \inf_{g_{\theta}\in\mathcal{G}_{NN},e_{\varphi}\in\mathcal{E}_{NN}}\sup_{f_{\omega} \in\mathcal{F}_{NN}}\frac{1}{n}\sum_{i=1}^{n}\Big(f_{\omega}(g_{\theta}(z_i),z_i)-f_{\omega}(x_i,e_{\varphi}(x_i)) \Big)+2\mathcal{E}_1\\
    &= 2\mathcal{E}_1+\mathcal{E}_2.
\end{align*}

\end{proof}
Note that we cannot directly apply the symmetrization technic (see appendix) to $\cE_3$ and $\cE_4$ since $e^*$ and $g^*$ are correlated with $x_i$ and $z_i$. However, this problem can be  solved by replacing the samples $(x_i,z_i)$ in the empirical terms in $\cE_3$ and $\cE_4$ with ghost samples $(x'_i,z'_i)$ independent of $(x_i,z_i)$ and replacing $g^*$ and $e^*$ with $g^{**}$ and $e^{**}$ which are obtained from the ghost samples, respectively. That is, we replace $\hat{\mathbb{E}} f_{\omega}(g^*(z),z)$ and $\hat{\mathbb{E}} f_{\omega}(x,e^*(x))$ with $\hat{\mathbb{E}} f_{\omega}(g^{**}(z'),z') $  and $\hat{\mathbb{E}} f_{\omega}(x',e^{**}(x'))$ in $\cE_3$ and $\cE_4$, respectively. Then we can proceed with the same proof of Lemma \ref{lma4} and apply the symmetrization technic  to $\cE_3$ and $\cE_4,$ since $(g^*(z_i),z_i)$ and $(g^{**}(z'_i),z'_i)$ have the same distribution. To simplify the notation, we will just use $\hat{\mathbb{E}} f_{\omega}(g^*(z),z)$ and $\hat{\mathbb{E}} f_{\omega}(x,e^*(x))$ to denote $\hat{\mathbb{E}} f_{\omega}(g^{**}(z'),z') $  and $\hat{\mathbb{E}} f_{\omega}(x',e^{**}(x'))$ here, respectively.

\subsection{Approximation errors}
We now discuss the errors due to the discriminator approximation and the generator and
 encoder approximation.

\subsubsection{The discriminator approximation error $\mathcal{E}_1$} \label{sec41}
The discriminator approximation error
$\mathcal{E}_1$
describes how well the discriminator neural network class approximates functions from
the Lipschitz class $\mathcal{F}^1$.
Lemma \ref{lma5} below can be applied  to obtain the 
 neural network approximation error for Lipschitz functions. It leads to a quantitative and non-asymptotic approximation rate in terms of the width and depth of the neural networks when bounding
 $\mathcal{E}_1$.

\begin{lemma}[Shen et al. (2021)]\label{lma5}
Let $f$ be a Lipschitz continuous function defined on $[-R,R]^d$. For arbitrary $W, L\in \mathbb{N}_+$, there exists a function $\psi$ implemented by a ReLU feedforward neural network with width $W$ and depth $L$ such that
\begin{align*}
    ||f-\psi||_{\infty}= O\big(\sqrt{d}R (WL)^{-\frac{2}{d}}\big).
\end{align*}
\end{lemma}

By Lemma \ref{lma5} and our choice of the architecture of discriminator class $\mathcal{F}_{NN}$ in the theorems, we have $\mathcal{E}_1=O\big(\sqrt{d} (W_1 L_1)^{-\frac{2}{d+1}}\log n\big)$.  Theorem \ref{lma5} also informs about how to choose the architecture of the discriminator networks based on how small we want the approximation error $\mathcal{E}_1$ to be.
By setting  $(W_1 L_1)^2\ge n$, $\mathcal{E}_1$ is dominated by the stochastic terms $\mathcal{E}_3$ and $\mathcal{E}_4$. 

\subsubsection{The generator and encoder approximation error $\mathcal{E}_2$} \label{sec42}
The generator and encoder approximation error $\mathcal{E}_2$
describes how powerful the generator and encoder classes are in pushing the empirical distributions $\hat{\mu}_n$ and $\hat{\nu}_n$ to each other.
A natural question is
 \begin{itemize}
    \item Can we find some generator and encoder neural network functions such that $\mathcal{E}_2=0$?
\end{itemize}
Most of the current literature concerning the error analysis of GANs applied the optimal transport theory \citep{villani2008optimal}  to minimize an error term similar to $\mathcal{E}_2$, see, for example,  \citet{zhao}. However,
 the existence of the optimal transport function from $\mathbb{R}\to\mathbb{R}^d$ is not guaranteed. Therefore, the existing analysis of GANs can only deal with the scenario 
 when  the reference and the target data distribution are assumed to have the same dimension.
This equal dimensionality assumption is not satisfied in the actual training of GANs or bidirectional GANs
in many applications.
Here, instead of using the optimal transport theory, we establish the following approximation results in Theorem \ref{lma6}, which enables us to forgo the equal dimensionality assumption.


\begin{theorem}\label{lma6}
Suppose that $\nu$ supported on $\mathbb{R}$ and $\mu$ supported on $\mathbb{R}^d$ are both absolutely continuous w.r.t. the Lebesgue measures, and $z_i's$ and $x_i's$ are i.i.d. samples from $\nu$ and $\mu$, respectively for $1\leq i\leq n$. Then there exist generator and encoder neural network functions $g: \mathbb{R}\mapsto\mathbb{R}^d$ and $e: \mathbb{R}^d\mapsto\mathbb{R}$ such that $g$ and $e$ are inverse bijections of each other between $\{z_i: 1\le i\le n\}$ and $\{x_i: 1\le i\le n\}$ up to a permutation. Moreover, such neural network functions $g$ and $e$ can be obtained by properly specifying $W_2^2 L_2= c_2 d n $ and $W_3^2 L_3= c_3 n$ for some constant $12\le c_2, c_3\le 384$. 
\end{theorem}
\begin{proof}
By the absolute continuity of $\nu$ and $\mu$, all the $z_i's$ and $x_i's$ are distinct a.s.. We can reorder $z_i's$ from the smallest to the largest, so $z_1<z_2<\ldots<z_n$. Let $z_{i+1/2}$ be any point between $z_i$ and $z_{i+1}$ for $i\in\{1,2,\ldots,n-1\}$. We define the continuous piece-wise linear function $g: \mathbb{R}\mapsto\mathbb{R}^d$ by
\begin{align*}
    g(z)=
    \begin{cases}
    x_1 & z\leq z_{1},\\
    \frac{z-z_{i+1/2}}{z_i-z_{i+1/2}}x_i+\frac{z-z_{i}}{z_{i+\frac{1}{2}}-z_{i}}x_{i+1} & z=(z_i, z_{i+1/2}), \text{ for } i=1,\ldots,n-1,\\
    x_{i+1} & z\in[z_{i+1/2},z_{i+1}], \text{ for } i=1,\ldots,n-2,\\
    x_n & z\geq z_{n-1+1/2}.
    \end{cases}
\end{align*}
By \citet[Lemma 3.1]{yang2021capacity} , $g\in \mathcal{NN}(W_2,L_2)$ if $n\le (W_2-d-1)\floor*{\frac{W_2-d-1}{6d}}\floor*{\frac{L_2}{2}}$. Taking $n= (W_2-d-1)\floor*{\frac{W_2-d-1}{6d}}\floor*{\frac{L_2}{2}}$, a simple calculation shows $W^2_2 L_2 = cdn$ for some constant $12\le c\le 384$.  The existence of neural net function $e$ can be constructed in the same way due to the fact that the first coordinate of $x_i's$ are distinct almost surely.
\end{proof}

When the number of point masses of the empirical distributions are relatively moderate compared with the structure of the neural nets, we can approximate empirical distributions arbitrarily well with any empirical distribution with the same number of point masses pushforwarded by the neural nets.

Theorem \ref{lma6} provides an effective 
 way to specify the architecture of generator and encoder classes. According to this lemma, we can take $n=\frac{W_2-d}{2}\floor*{\frac{W_2-d}{6d}}\floor*{\frac{L_2}{2}}+2= \frac{W_3-1}{2}\floor*{\frac{W_3-1}{6}}\floor*{\frac{L_3}{2}}+2$, which gives rise to $W_2^2 L_2/d\asymp W_3^2 L_3\asymp n $. 
More importantly, Theorem \ref{lma6} can be applied to bound $\mathcal{E}_2$ as follows.
\begin{align*}
\mathcal{E}_2=&\inf_{g_{\theta}\in\mathcal{G}_{NN},e_{\varphi}\in\mathcal{E}_{NN}}\sup_{f_{\omega} \in\mathcal{F}_{NN}}\frac{1}{n}\sum_{i=1}^{n}\Big(f_{\omega}(g_{\theta}(z_i),z_i)-f_{\omega}(x_i,e_{\varphi}(x_i)) \Big)\\
\le &\inf_{g_{\theta}\in\mathcal{G}_{NN}}\sup_{f_{\omega} \in\mathcal{F}_{NN}}\frac{1}{n}\sum_{i=1}^{n}\Big(f_{\omega}(g_{\theta}(z_i),z_i)-f_{\omega}(x_i,z_i)\Big)\\
&+\inf_{e_{\varphi}\in\mathcal{E}_{NN}}\sup_{f_{\omega} \in\mathcal{F}_{NN}}\frac{1}{n}\sum_{i=1}^{n}\Big(f_{\omega}(x_i,z_i)-f_{\omega}(x_i,e_{\varphi}(x_i))\Big)\\
=&0.
\end{align*}
We simply reordered $z_i's$ and $x_i's$ as in the proof.
Therefore, this error term can be perfectly eliminated.

\subsection{Stochastic errors}
The stochastic error $\mathcal{E}_3$ ($\mathcal{E}_4$)
quantifies how close the empirical distribution and the true latent joint distribution (data joint distribution) are with the Lipschitz class $\mathcal{F}^1$ as the evaluation class under IPM. We apply the results in the refined Dudley inequality  \citep{schreuder2020bounding}
in Lemma \ref{lma7} to bound $\mathcal{E}_3$ and $\mathcal{E}_4$.
\begin{lemma}[Refined Dudley Inequality]\label{lma7}
For a symmetric function class $\mathcal{F}$ with $\sup_{f\in\mathcal{F}} ||f||_{\infty}\leq M$, we have
\begin{align*}
    \mathbb{E}[d_{\mathcal{F}}(\hat{\mu}_n,\mu)]\leq \inf_{0<\delta<M}\left(4\delta+\frac{12}{\sqrt{n}}\int_{\delta}^M\sqrt{\log \mathcal{N}(\epsilon,\mathcal{F},||\cdot||_{\infty})}d\epsilon\right).
\end{align*}
\end{lemma}
The original Dudley inequality \citep{dudley, wellner} suffers from the problem that if the covering number $\mathcal{N}(\epsilon, \mathcal{F}, ||\cdot||_{\infty})$ increases too fast as $\epsilon$ goes to $0$, then the upper bound will be infinity, which is totally meaningless. The improved Dudley inequality circumvents such a problem by only allowing $\epsilon$ to integrate from $\delta>0$ as is shown in Lemma \ref{lma7}, which also indicates that $\mathbb{E}\mathcal{E}_3$ scales with the covering number $\mathcal{N}(\epsilon,\mathcal{F}^1, ||\cdot||_{\infty})$.

By calculating the covering number of $\mathcal{F}^1$ and utilizing
the refined Dudley inequality,
we can obtain the upper bound
\begin{align}
    \max\{\mathbb{E}\mathcal{E}_3, \mathbb{E}\mathcal{E}_4\}&= O\left( C_d n^{-\frac{1}{d+1}}\log n\wedge \sqrt{d} n^{-\frac{1}{d+1}}(\log n)^{1+\frac{1}{d+1}}\right).
\label{eq:4.6}
\end{align}

\section{Related work}
\label{sec11}

Recently, several impressive works have studied the challenging problem of the convergence properties of unidirectional GANs. \citet{arora} noted that training of GANs may not have good generalization properties in the sense that even if training may appear successful but the trained distribution may be far from target distribution in standard metrics. On the other hand, \citet{bai2018} showed that GANs can learn distributions
in Wasserstein distance with polynomial sample complexity.
 \citet{liang2020} studied the rates of convergence of a class of GANs,  including Wasserstein, Sobolev and MMD GANs.  This work also established the nonparametric minimax optimal rate under the Sobolev IPM.
 The results of \citet{bai2018}  and \citet{liang2020} require invertible generator networks, meaning all the weight matrices need to be full-rank, and the activation function needs to be the invertible leaky ReLU activation. \citet{zhao} established an upper bound for the estimation error rate  under H\"older evaluation and target density classes, where $\mathcal{H}^{\beta}$ is H\"older class with regularity $\beta$ and the density of the target $\mu$ is assumed to belong to $\mathcal{H}^{\alpha}$.
They assumed that the reference distribution has the same dimension as the target distribution and applied the optimal transport theory to control the generator approximation error.
However, how the prefactor depends in the error bounds on the dimension $d$ in the existing results \citep{liang2020, zhao} is either not clearly described or is exponential.
In high-dimensional settings with large $d$, this makes a substantial difference in the quality of the error bounds.

\citet{singhnonparametric} studied minimax convergence rates of nonparametric density estimation under a class of adversarial losses and investigated how the choice of loss and the assumed smoothness of the underlying density together determine the minimax rate; they also discussed
connections to learning generative models in a minimax statistical sense.
\citet{uppal2019nonparametric} generates the idea of Sobolev IPM to Besov IPM, where both target density and the evaluation classes  are Besov classes. They also showed
 how their results imply bounds on the statistical error of a GAN.

These results provide important insights in the understanding of GANs. However, as we mentioned
earlier,  some of the assumptions made in these results, including equal dimension between the reference and target distributions and bounded support of the distributions, are not satisfied in the training of GANs in practice.
Our results avoid these assumptions. Moreover, the prefactors in our error bounds are clearly described as being dependent on the square root of the dimension $d$. Finally, the aforementioned results only dealt with unidirectional GANs. Our work is the first to address the convergence properties of bidirectional GANs.



\section{Conclusion}
\label{conclusion}
This paper derives the error bounds for the bidirectional GANs under the Dudley distance between the latent joint distribution and the data joint distribution. The results are established without the two crucial conditions that are commonly assumed in the existing literature: equal dimensionality between the reference and the target distributions and bounded support for these distributions.
Additionally, this work contributes to the neural network approximation theory by
constructing neural network functions such that the pushforward distribution of an empirical distribution can perfectly approximate another arbitrary empirical distribution with a different dimension as long as their number of point masses are equal. A novel decomposition of integral probability metric is also developed for error analysis of bidirectional GANs, which can be useful in other generative learning problems.

{\color{black}
A limitation of our results, as well as all the existing results on the convergence properties of GANs, is that they suffer from the curse of dimensionality, which cannot be circumvented by assuming sufficient smoothness assumptions. In many applications, high-dimensional complex data such as images, texts and natural languages, tend to be supported on approximate lower-dimensional manifolds. It is desirable
to take into such structure in the theoretical analysis. An important extension of the present results is to show that bidirectional GANs can circumvent the curse of dimensionality if the target distribution is assumed to be supported on an approximate lower-dimensional manifold. This appears to be a technically challenging problem and will be pursued in our future work.
}

\section*{Acknowledgements}
The authors wish to thank the three anonymous reviewers for
their insightful comments and constructive suggestions that helped improve the paper significantly.

The work of J. Huang is partially supported by the U.S. NSF grant DMS-1916199.
The work of Y. Jiao is supported in part by the National Science Foundation of China
under Grant 11871474 and by the research fund of KLATASDSMOE. The work of Y. Wang is
supported in part by the Hong Kong Research Grant Council grants 16308518 and 16317416
and HK Innovation Technology Fund ITS/044/18FX, as well as Guangdong-Hong Kong-Macao
Joint Laboratory for Data-Driven Fluid Mechanics and Engineering Applications.

\newpage

\bibliographystyle{apa}
\bibliography{bigan_bib}

\newpage

\appendix
\numberwithin{equation}{section}
\makeatletter
\newcommand{\section@cntformat}{Appendix \thesection:\ }

\makeatother


\noindent
\textbf{\large Appendix}

In the appendix, we first  prove Theorem \ref{thm1}, and then Theorems \ref{cor1} and \ref{cor2}.

\section{Notations and Preliminaries}\label{sec12}

We use $\sigma$ to denote the ReLU activation function in neural networks, which is $\sigma(x)=\max\{x,0\}$. Without further indication, $\|\cdot\|$ represents the $L_2$ norm. For any function $g$, let $\|g\|_{\infty}=\sup_{x}\|g(x)\|$.
We use notation $O(\cdot)$ and $\tilde{O}(\cdot)$ to express the order of function slightly differently, where $O(\cdot)$ omits the universal constant not relying on $d$ while $\tilde{O}(\cdot)$ omits the constant related to $d$. We use $B_2^{d}(a)$ to denote $L_2$ ball in $\mathbb{R}^d$ with center at $\mathbf{0}$ and radius $a$.
Let $g_{\#}\nu$ be the pushforward distribution of $\nu$ by function $g$ in the sense that $g_{\#}\nu(A)=\nu(g^{-1}(A))$ for any measurable set $A$.


The $r$-\textbf{covering number} of some class $\mathcal{F}$ w.r.t. norm $\|\cdot\|$ is the minimum number of $r$-$\|\cdot\|$ radius balls needed to cover $\mathcal{F}$, which we denote as $\mathcal{N}(r, \mathcal{F},\|\cdot\|)$. We denote $\mathcal{N}(r,\mathcal{F},L_2(P_n))$ as the covering number of $\mathcal{F}$ w.r.t. $L_2(P_n)$, which is defined as $\|f\|^2_{L_2(P_n)}=\frac{1}{n}\sum_{i=1}^n \|f(X_i)\|^2$ where $X_1,\ldots,X_n$ are the empirical samples. We denote $\mathcal{N}(r,\mathcal{F},L_{\infty}(P_n))$ as the covering number of $\mathcal{F}$ w.r.t. $L_{\infty}(P_n)$, which is defined as $\|f\|_{L_{\infty}(P_n)}= \max_{1\leq i\leq n}\|f(X_i)\|$. It is easy to check that
\begin{align*}
    \mathcal{N}(r,\mathcal{F},L_2(P_n)) \leq \mathcal{N}(r,\mathcal{F},L_{\infty}(P_n)) \leq \mathcal{N}(r, \mathcal{F},\|\cdot\|_{\infty}).
\end{align*}

\section{Restriction on the domain of uniformly bounded Lipschitz function class $\mathcal{F}^1$}\label{sec23}
So far, most of the related works
assume that the target distribution $\mu$ is  supported on a compact set, for example \cite{zhao} and \cite{liang2020}. To remove the compact support assumption, we need to assume Assumption 1, i.e.,  the tails of the target $\mu$ and the reference $\nu$ are subexponential. Define $\mathcal{F}_n^1:=\{f|_{B_2^{d+1}(\sqrt{2}\log n)}:f\in\mathcal{F}^1\}$.
{ In this section, we show that proving Theorem \ref{thm1} is equivalent to establishing the same convergence rate but with the domain restricted function class $\mathcal{F}_n^1$ as the evaluation class}.

Under Assumption \ref{asp1} and by the Markov inequality, we have
\begin{align}
    P_{\nu}(\|z\|>\log n)\leq \frac{\mathbb{E}_{\nu} \|z\| \mathbbm{1}_{\{\|z\|>\log n\}}}{\log n}= O( n^{-\frac{(\log n)^{\delta}}{d}}/\log n)
    \label{eq:2.3}
\end{align}

The Dudley distance between latent joint distribution $\hbnu$ and data joint distribution $\hbmu$ is defined as
\begin{align}
d_{\mathcal{F}^1}(\hbnu,\hbmu)=\sup_{f\in\mathcal{F}^1}\mathbb{E} f(\hat{g}(z),z)-\mathbb{E} f(x,\hat{e}(x))
\label{eq:2.4}
\end{align}
The first term above can be decomposed as
\begin{align}
  \mathbb{E} f(\hat{g}(z),z)&= \mathbb{E} f(\hat{g}(z),z)\mathbbm{1}_{\|z\|\le\log n}+\mathbb{E} f(\hat{g}(z),z)\mathbbm{1}_{\|z\|>\log n}
  \label{eq:2.5}
\end{align}
For any $f\in \mathcal{F}^1$ and fixed point $z_0$ such that $\|z_0\|\le \log n$, due to the Lipschitzness of $f$, the second term above satisfies
\begin{align*}
    |\mathbb{E} f(\hat{g}(z),z)\mathbbm{1}_{\|z\|>\log n}|&\le |\mathbb{E} f(\hat{g}(z),z)\mathbbm{1}_{\|z\|>\log n}-\mathbb{E} f(\hat{g}(z_0),z_0)\mathbbm{1}_{\|z\|>\log n}|\\
    &+|\mathbb{E} f(\hat{g}(z_0),z_0)\mathbbm{1}_{\|z\|>\log n}|\\
    \le & \mathbb{E}\|(\hat{g}(z)-\hat{g}(z_0),z-z_0) \|\mathbbm{1}_{\|z\|>\log n}+B  P_{\nu}(\|z\|>\log n)\\
    \le & \mathbb{E}(\|(\hat{g}(z)-\hat{g}(z_0)\|+ \|z-z_0 \|)\mathbbm{1}_{\|z\|>\log n}+B  P_{\nu}(\|z\|>\log n)\\
    \le & 2(\log n) P_{\nu}(\|z\|>\log n)+ \mathbb{E}\|z-z_0 \|\mathbbm{1}_{\|z\|>\log n}+B  P_{\nu}(\|z\|>\log n)\\
    = & O(n^{-\frac{(\log n)^{\delta}}{d}})
\end{align*}
where the second inequality is due to lipschitzness and boundedness of $f$, and the last inequality is due to Assumption \ref{asp1}, \eqref{eq:2.3}, and the boundedness condition of $\hat{g}$.
In the first term in \eqref{eq:2.5}, $f$ only acts on the increasing $L_2$ ball $B^d_2(\sqrt{2}\log n)$ because of Condition \ref{cond:1} and the indicator function $\mathbbm{1}_{\{\|z\|\leq \log n\}}$. Similarly, we can apply the same procedure to the second term in \eqref{eq:2.4}. Therefore, it is still an equivalent problem if we restrict the domain of $\mathcal{F}^1$ on $B^d_2(\sqrt{2}\log n)$.
Hence, in order to prove the estimation error rate in Theorem \ref{thm1},
we only need to show that for the restricted evaluation function class $\mathcal{F}^1_n$, we have
\begin{align*}
    \mathbb{E}d_{\mathcal{F}_n^1}(\hbnu,\hbmu)\leq C_0 \sqrt{d}n^{-\frac{1}{d+1}}(\log n)^{1+\frac{1}{d+1}}\wedge C_d n^{-\frac{1}{d+1}}\log n
\end{align*}

Due to this fact, to keep notation simple,
we are going to denote $\mathcal{F}_n^1$ as $\mathcal{F}^1$ in the following sections.
\begin{remark}
The restriction on $\mathcal{F}^1$ is technically necessary for calculating the covering number of $\mathcal{F}^1$ later we will see 
the use of it when bounding the stochastic error $\mathcal{E}_3$ and $\mathcal{E}_4$ below.
\end{remark}

\section{Stochastic errors}

\subsection{Bounding $\mathcal{E}_3$ and $\mathcal{E}_4$} \label{sec43}
The stochastic errors $\mathcal{E}_3$  and $\mathcal{E}_4$
quantify how close the empirical distributions and the true latent joint distribution (data joint distribution) are with the Lipschitz class $\mathcal{F}^1$ as the evaluation class under IPM. We apply the results in Lemma \ref{lma7} to bound $\mathcal{E}_3$ and $\mathcal{E}_4$.
We introduce two methods to bound $\max\{\mathcal{E}_3, \mathcal{E}_4\}$, which gives two different upper bounds for $\max\{\mathcal{E}_3, \mathcal{E}_4\}$. They both utilize the following lemma, which we shall prove later. 
More detailed description about the refined Dudley inequality can be found in  \citet{srebro2010note} and \citet{schreuder2020bounding}.

\begin{lemma}[Refined Dudley Inequality]\label{lma7}
For a symmetric function class $\mathcal{F}$ with \\
$\sup_{f\in\mathcal{F}} \|f\|_{\infty}\leq M$, we have
\begin{align*}
    \mathbb{E}[d_{\mathcal{F}}(\hat{\mu}_n,\mu)]\leq \inf_{0<\delta<M}\left(4\delta+\frac{12}{\sqrt{n}}\int_{\delta}^M\sqrt{\log \mathcal{N}(\epsilon,\mathcal{F},\|\cdot\|_{\infty})}\, d\epsilon\right).
\end{align*}
\end{lemma}
\begin{remark}
The original Dudley inequality ~\citep{dudley, wellner} suffers from the problem that if the covering number $\mathcal{N}(\epsilon, \mathcal{F}, \|\cdot\|_{\infty})$ increases too fast as $\epsilon$ goes to $0$, then the upper bound can be infinite. 
The improved Dudley inequality circumvents this problem by only allowing $\epsilon$ to integrate from $\delta>0$, 
which also indicates that $\mathbb{E}\mathcal{E}_3$ scales with the covering number $\mathcal{N}(\epsilon,\mathcal{F}^1, \|\cdot\|_{\infty})$.
\end{remark}
\subsubsection{The first method (explicit constant)}\label{sec431}
The first method provides an explicit constant depending on $d$ at the expense of the higher order of $\log n$ in the upper bounds. It utilizes the next lemma
\citep[Lemma 6]{gottlieb2013efficient}, which turns the problem of bounding the covering number of a Lipschitz function class into the one bounding the covering number of the domain defined for the function class.

\begin{lemma}[\citet{gottlieb2013efficient}]\label{lma8}
Let $\mathcal{F}^L$ be the collection of $L-$Lipschitz functions mapping the metric space $(\mathcal{X},\rho)$ to $[0,1]$. Then the covering number of $\mathcal{F}^L$ can be estimated in terms of the covering number of $\mathcal{X}$ with respect to $\rho$ as follows.

\begin{align*}
    \mathcal{N}(\epsilon,\mathcal{F}^L,\|\cdot\|_{\infty})\leq (\frac{8}{\epsilon})^{\mathcal{N}(\epsilon/8L,\mathcal{X},\rho)}.
\end{align*}
\end{lemma}
Now we apply Lemma \ref{lma8} to bound the covering number for the 1-Lipschitz class $\mathcal{N}(\epsilon,\mathcal{F}^1, \|\cdot\|_{\infty})$ by bounding the covering number for its domain  $\mathcal{N}(\epsilon,B^{d+1}_2(\sqrt{2}\log n), \|\cdot\|_{2})$. Define a new function class $\mathcal{F}^{\frac{1}{2B}}$ as
\begin{align*}
    \mathcal{F}^{\frac{1}{2B}}:=\{\frac{f+B}{2B}:f\in \mathcal{F}^1\}.
\end{align*}
Recall that $\mathcal{F}^1$ is restricted on $B^{d+1}_2(\sqrt{2}\log n)$. Obviously, $\mathcal{F}^{\frac{1}{2B}}$ is a $\frac{1}{2B}-$Lipschitz function class : $B^{d+1}_2(\sqrt{2}\log n)\mapsto [0,1]$. A direct application of Lemma \ref{lma8} shows that
\begin{align}
\mathcal{N}(\epsilon,\mathcal{F}^{\frac{1}{2B}}, \|\cdot\|_{\infty})\leq \left(\frac{8}{\epsilon}\right)^{\mathcal{N}(\epsilon B/4, B_2^{d+1}(\sqrt{2}\log n),\|\cdot\|_2)}.
 \label{eq:4.2}
\end{align}
By the definition of $\mathcal{F}^{\frac{1}{2B}}$, the covering numbers satisfy
\begin{align}
\mathcal{N}(2B\epsilon,\mathcal{F}^1,\|\cdot\|_{\infty})
=\mathcal{N}(\epsilon,\mathcal{F}^{\frac{1}{2B}}, \|\cdot\|_{\infty}).
\label{eq:4.3}
\end{align}
Note that $B^{d+1}_2(\sqrt{2}\log n)$ is a subset of $[-\sqrt{2}\log n,\sqrt{2}\log n]^d$, and  $[-\sqrt{2}\log n,\sqrt{2}\log n]^d$ can be covered with finite $\epsilon$-balls in $\mathbb{R}^d$ that cover the small hypercube with side length $2\epsilon/\sqrt{d}$.
It follows that
\begin{align}
\mathcal{N}(\epsilon, B^{d+1}_2(\sqrt{2}\log n), \|\cdot\|_2)\leq \left(\frac{\sqrt{2(d+1)}\log n}{\epsilon}\right)^{d+1}.
\label{eq:4.4}
\end{align}
Combining \eqref{eq:4.2}, \eqref{eq:4.3} and  \eqref{eq:4.4}, we obtain an upper bound for the covering number of the 1-Lipschitz class
$\mathcal{F}^1$
\begin{align}
\log \mathcal{N}(\epsilon,\mathcal{F}^1, \|\cdot\|_{\infty})\leq \left(\frac{8\sqrt{2(d+1)}\log n}{\epsilon}\right)^{d+1}\log \frac{16 B}{\epsilon}.
\label{eq:4.5}
\end{align}
With the upper bound for the covering entropy in \eqref{eq:4.5}, a direct application of Lemma \ref{lma7} (see Section \ref{ap:d} for details) by taking $\delta=8\sqrt{2(d+1)} n^{-\frac{1}{d+1}}(\log n)^{1+\frac{1}{d+1}}$ leads to
\begin{align}
\max\{\mathbb{E}\mathcal{E}_3,\mathbb{E}\mathcal{E}_4\}&=O \left( \sqrt{d}n^{-\frac{1}{d+1}}(\log n)^{1+\frac{1}{d+1}} + n^{-\frac{1}{d+1}}(\log n)^{1+\frac{1}{d+1}}\right)\\
& = O\left( \sqrt{d} n^{-\frac{1}{d+1}}(\log n)^{1+\frac{1}{d+1}}\right).
\label{eq:4.6a}
\end{align}

\subsubsection{The second method (better order of $\log n$)}\label{sec432}
We now consider the second method that leads to a better order for the $\log n$ term in the upper bound at the expense of explicitness of the constant related to $d$.  The next lemma directly provides an upper bound for the covering number of Lipschitz class but with an implicit constant related to $d$. It is a straightforward corollary of \citet[Theorem 2.7.1]{wellner}.

\begin{lemma}\label{lma9}
 Let $\mathcal{X}$ be a bounded, convex subset of $\mathbb{R}^d$ with nonempty interior. There exists a constant $c_d$ depending only on $d$ such that
 \begin{align*}
     \log \mathcal{N}(\epsilon, \mathcal{F}^1(\mathcal{X}),\|\cdot\|_{\infty})\leq c_d \lambda(\mathcal{X}^1)\left(\frac{1}{\epsilon}\right)^{d}
 \end{align*}
for every $\epsilon>0$, where $\mathcal{F}^1(\mathcal{X})$ is the 1-Lipschitz function class defined on $\mathcal{X}$, and $\lambda(\mathcal{X}^1)$ is the Lebesgue measure of the set $\{x:\|x-\mathcal{X}\|<1\}$.
\end{lemma}

Applying Lemmas \ref{lma7} and \ref{lma9} (see Section \ref{ap:d} for details) by taking $\delta=n^{-\frac{1}{d+1}}\log n$ yields
\begin{align}
    \max\{\mathbb{E}\mathcal{E}_3,\mathbb{E}\mathcal{E}_4\}&= O\left( C_d n^{-\frac{1}{d+1}}\log n\right),
  \label{eq:4.6b}
\end{align}
where $C_d$ is some constant depending on $d$. Combining (\ref{eq:4.6a}) and (\ref{eq:4.6b}),
we get
\begin{align}
    \max\{\mathbb{E}\mathcal{E}_3,\mathbb{E}\mathcal{E}_4\}&= O\left( C_d n^{-\frac{1}{d+1}}\log n\wedge \sqrt{d} n^{-\frac{1}{d+1}}(\log n)^{1+\frac{1}{d+1}}\right).
\label{eq:4.6}
\end{align}
\begin{remark}
Here, we have a tradeoff between the logarithmic factor $\log n$ and the explicitness of the constant depending on $d$. If we want an explicit constant depending on $d$, then we have the
factor $(\log n)^{1+\frac{1}{d+1}}$ in the upper bound. Later we will see that $\mathbb{E}\mathcal{E}_3$ and $\mathbb{E}\mathcal{E}_4$ are the dominating terms in the four error terms, hence the explicitness of the corresponding constant becomes
important. Therefore, we list two different methods here to bound $\mathbb{E}\mathcal{E}_3$ and $\mathbb{E}\mathcal{E}_4$.
\end{remark}

\subsection{Combination of the four error terms}\label{sec5}
With all the upper bounds for the four different error terms obtained above, next  we consider $\mathcal{E}_1$-$\mathcal{E}_4$ simultaneously to obtain an overall convergence rate.
First, recall how we bound $\mathcal{E}_1$ and $\mathcal{E}_4$. With Lemma \ref{lma5},  we have
\begin{align}
    \mathcal{E}_1=O\left( \sqrt{d} (W_1 L_1)^{-\frac{2}{d+1}}\log n\right).
\label{eq:5.1}
\end{align}
To control $\mathcal{E}_1$ while keeping the architecture of discriminator class $\mathcal{F}_{NN}$ as small as possible, we let $W_1 L_1=\ceil*{\sqrt{n}}$, so that $\mathcal{E}_1= O\left( \sqrt{d} n^{-\frac{1}{d+1}}\log n\right)$ dominated by $\mathcal{E}_3$ and $\mathcal{E}_4$.

By Theorem \ref{lma6}, we can choose the architectures of generator and encoder classes accordingly to perfectly control $\mathcal{E}_2$, i.e. $\mathcal{E}_2=0$.

We note that because we imposed Condition \ref{cond:1} on both generator and encoder classes, Theorem \ref{lma6} can not be applied if we have some $\|x_i\|$ or $\|z_i\|$ greater than $\log n$, in which case $\mathcal{E}_2$ can not be perfectly controlled. But we can still handle this case by considering the probability of the bad set.

Under Condition \ref{cond:1}, on the nice set $A:=\{\max_{1\leq i\leq n}\|x_i\|\leq \log n\}\cap \{\max_{1\leq i\leq n}\|z_i\|\leq \log n\}$, we have $\mathcal{E}_2=0$. Probability of the nice set $A$ has the following lower bound.
\begin{align*}
    P(A)&=P_{\mu}(||x_i||\leq \log n)^n\cdot P_{\nu}(||z_i||\leq \log n)^n\\
    & \geq  (1-C n^{-\frac{(\log n)^{\delta}}{d}})^{2n}, \  \text{ for some constant $C>0$ by Assumption \ref{asp1}}\\
    & \geq 1-C n^{-\frac{(\log n)^{\delta}}{d}}\cdot (2n), \ \text{  for large $n$.}
\end{align*}
The bad set $A^c$ is where  $\mathcal{E}_2> 0$, which has the probability upper bound as follows.
\begin{align*}
   P(A^c)&\leq C n^{-\frac{(\log n)^{\delta}}{d}}\cdot (2n)\\
   &=O\left( n^{-\frac{(\log n)^{\delta'}}{d}}\right) \text{, for any $\delta'<\delta$}.
\end{align*}
In Assumption \ref{asp1}, the $(\log n)^{\delta}$ factor was to make the tail of the target $\mu$ strictly subexponential, which leads to
$P(A^c)\to 0$, while the exponential tail or heavier will cause the undesired result $P(A^c)\to 1$.

Now we are ready to obtain the desired result in Theorem \ref{thm1}. The nice set $A=\{\max_{1\leq i\leq n}\|x_i\|\leq \log n\}\cap \{\max_{1\leq i\leq n}\|z_i\|\leq \log n\}$ is where $\mathcal{E}_2=0$.
By combining the results discussed above, we have
\begin{align*}
    \mathbb{E}d_{\mathcal{F}^1}(\hbnu,\hbmu)&= 2\mathcal{E}_1+\mathcal{E}_2\mathbbm{1}_{A}+\mathcal{E}_2\mathbbm{1}_{A^c}+\mathbb{E}\mathcal{E}_3+\mathbb{E}\mathcal{E}_4
    \\
    &\le O\left(\sqrt{d} n^{-\frac{1}{d+1}}\log n+0+2B P_{\mu}(A^c)+\sqrt{d} n^{-\frac{1}{d+1}}(\log n)^{1+\frac{1}{d+1}}\wedge C_d n^{-\frac{1}{d+1}}\log n\right)\\
    &= O\left(\sqrt{d} n^{-\frac{1}{d+1}}(\log n)^{1+\frac{1}{d+1}}\wedge C_d n^{-\frac{1}{d+1}}\log n+ n^{-\frac{(\log n)^{\delta'}}{d}}\right)\\
    &= O\left(\sqrt{d} n^{-\frac{1}{d+1}}(\log n)^{1+\frac{1}{d+1}}\wedge C_d n^{-\frac{1}{d+1}}\log n\right),
\end{align*}
which completes the proof of Theorem \ref{thm1}.

\section{Proof of Inequality (\ref{dIPM}) }
\label{ap:a}


For ease of reference, we restate inequality (\ref{dIPM}) as the following lemma.

\begin{customlemma}{4.2}\label{lma17}
For any symmetric function classes $\mathcal{F}$ and $\mathcal{H}$, denote the approximation error $\mathcal{E}(\mathcal{H},\mathcal{F})$ as
\begin{align*}
    \mathcal{E}(\mathcal{H},\mathcal{F}):=
    \underset{h\in\mathcal{H}}{\sup}\underset{f\in\mathcal{F}}{\inf}\|h-f\|_{\infty},
\end{align*}
then for any probability distributions $\mu$ and $\nu$,
\begin{align*}
    d_{\mathcal{H}}(\mu,\nu)-d_{\mathcal{F}}(\mu,\nu)\leq 2\mathcal{E}(\mathcal{H},\mathcal{F}).
\end{align*}
\end{customlemma}

\begin{proof}[Proof of Lemma \ref{lma17}]
By the definition of supremum, for any $\epsilon>0$, there exists $h_{\epsilon}\in\mathcal{H}$ such that
\begin{align*}
    d_{\mathcal{H}}(\mu,\nu):&=\underset{h\in\mathcal{H}}{\sup}[\mathbb{E}_{\mu}h-\mathbb{E}_{\nu}h]\\
    &\leq \mathbb{E}_{\mu}h_{\epsilon}-\mathbb{E}_{\nu}h_{\epsilon}+\epsilon\\
    &=\underset{f\in\mathcal{F}}{\inf}[\mathbb{E}_{\mu}(h_{\epsilon}-f)-\mathbb{E}_{\nu}(h_{\epsilon}-f)+\mathbb{E}_{\mu}(f)-\mathbb{E}_{\nu}(f)]+\epsilon\\
    &\leq 2\underset{f\in \mathcal{F}}{\inf}\|h_{\epsilon}-f\|_{\infty}+d_{\mathcal{F}}(\mu,\nu)+\epsilon\\
    &\leq 2\mathcal{E}(\mathcal{H},\mathcal{F})+d_{\mathcal{F}}(\mu,\nu)+\epsilon,
\end{align*}
where the last line is due to the definition of $\mathcal{E}(\mathcal{H},\mathcal{F})$.
\end{proof}

It is easy to check that if we replace $d_{\mathcal{H}}(\mu,\nu)$ by $\hat{d}_{\mathcal{H}}(\mu,\nu):=\underset{h\in\mathcal{H}}{\sup}[\hat{\mathbb{E}}_{\mu}h-\hat{\mathbb{E}}_{\nu}h]$, Lemma \ref{lma17} still holds.

\section{Bounding $\mathbb{E}\mathcal{E}_3$ and $\mathbb{E}\mathcal{E}_4$}\label{ap:d}
\subsection{Method One}
With the upper bound for the covering entropy \eqref{eq:4.5}, i.e.
\begin{align*}
    \log \mathcal{N}(\epsilon,\mathcal{F}^1, \|\cdot\|_{\infty})\leq \left(\frac{8\sqrt{2(d+1)}\log n}{\epsilon}\right)^{d+1}\log \frac{16 B}{\epsilon}
\end{align*}
and $\delta=8\sqrt{2(d+1)} n^{-\frac{1}{d+1}}(\log n)^{1+\frac{1}{d+1}}$, applying Lemma \ref{lma7} we have
\begin{align*}
    \mathbb{E}\mathcal{E}_3&=O\left(\delta+n^{-\frac{1}{2}}\int_{\delta}^B \left(\frac{8\sqrt{2(d+1)}\log n}{\epsilon}\right)^{\frac{d+1}{2}}\left(\log\frac{16B}{\epsilon}\right)^{\frac{1}{2}}d\epsilon\right)\\
    & = O\left(\delta+n^{-\frac{1}{2}}(8\sqrt{2(d+1)}\log n)^{\frac{d+1}{2}}(\frac{\log n}{d+1})^{\frac{1}{2}}\delta^{1-\frac{d+1}{2}}\right)\\
    & = O\left(\sqrt{d} n^{-\frac{1}{d+1}}(\log n)^{1+\frac{1}{d+1}}+ n^{-\frac{1}{d+1}}(\log n)^{1+\frac{1}{d+1}}\right)\\
    &= O\left(\sqrt{d} n^{-\frac{1}{d+1}}(\log n)^{1+\frac{1}{d+1}}\right),
\end{align*}
where the second equality is due to
\begin{align*}
    \log \frac{16B}{\epsilon}&=O\left(\log \frac{1}{\epsilon}\right)=O\left(\log \left(\frac{n^{\frac{1}{d+1}}}{8\sqrt{2(d+1)}(\log n)^{1+\frac{1}{d+1}}}\right)\right)=O\left(\log n^{\frac{1}{d+1}}\right),
\end{align*}
and the third equality follows from simple algebra.

\subsection{Method Two}

By Lemma \ref{lma9}, we have
\begin{align*}
    \log \mathcal{N}(\epsilon,\mathcal{F}^1,\|\cdot\|_{\infty})\leq c_d\left(\frac{\log n}{\epsilon}\right)^{d+1}.
\end{align*}
Taking $\delta=n^{-\frac{1}{d+1}}\log n$ and applying Lemma \ref{lma7}, we obtain
\begin{align*}
    \mathbb{E}\mathcal{E}_3&=O\left(\delta+(\frac{c_d}{n})^{\frac{1}{2}}(\log n)^{\frac{d+1}{2}}\int_{\delta}^M (\frac{1}{\epsilon})^{\frac{d+1}{2}} d\epsilon \right)\\
    &= \tilde{O}\left(\delta+n^{-\frac{1}{2}}(\log n)^{\frac{d+1}{2}}\delta^{1-\frac{d+1}{2}}\right)\\
    &= \tilde{O}\left(n^{-\frac{1}{d+1}}\log n\right),
\end{align*}
where $\tilde{O}(\cdot)$ omitted the constant related to $d$.

\section{Proof of Lemma \ref{lma7}}\label{ap:e}
For completeness we provide a proof of the refined Dudley's inequality in Lemma \ref{lma7}.
We apply the standard symmetrization and chaining technics in the proof, see, for example,  \cite{wellner}.
\begin{proof}
Let $Y_1,\ldots,Y_n$ be random samples from $\mu$ which are independent of $X_i's$. Then we have
\begin{align*}
\mathbb{E}d_{\mathcal{F}}(\hat{\mu}_n,\mu)&=\mathbb{E}\underset{f\in\mathcal{F}}{\sup}[\frac{1}{n}\sum_{i=1}^n f(X_i)-\mathbb{E}f(X_i)]\\
&=\mathbb{E}\underset{f\in\mathcal{F}}{\sup}[\frac{1}{n}\sum_{i=1}^n f(X_i)-\mathbb{E}\frac{1}{n}\sum_{i=1}^n f(Y_i)]\\
&\leq \mathbb{E}_{X, Y}\underset{f\in\mathcal{F}}{\sup}[\frac{1}{n}\sum_{i=1}^n f(X_i)-\frac{1}{n}\sum_{i=1}^n f(Y_i)]\\
&=\mathbb{E}_{X, Y}\underset{f\in\mathcal{F}}{\sup}[\frac{1}{n}\sum_{i=1}^n \epsilon_i(f(X_i)-f(Y_i))]\\
&\leq 2\mathbb{E} \hat{\mathcal{R}}_n(\mathcal{F})
\end{align*}
where the first inequality is due to Jensen inequality, and the third equality is because that $(f(X_i)-f(Y_i))$ has symmetric distribution.

Let $\alpha_0=M$ and for any $j\in \mathbb{N}_+$ let $\alpha_j=2^{-j}M$. For each $j$, let $T_i$ be a $\alpha_i$-cover of $\mathcal{F}$ w.r.t. $L_2(P_n)$ such that $|T_i|=\mathcal{N}(\alpha_i,\mathcal{F},L_2(P_n))$. For each $f\in\mathcal{F}$ and $j$, pick a function $\hat{f}_i\in T_i$ such that $\|\hat{f}_i-f\|_{L_2(P_n)}<\alpha_i$. Let $\hat{f}_0=0$ and for any $N$, we can express $f$ by chaining as
\begin{align*}
    f=f-\hat{f}_N+\sum_{i=1}^N(\hat{f}_i-\hat{f}_{i-1}).
\end{align*}
Hence for any $N$, we can express the empirical Rademacher complexity as
\begin{align*}
    \hat{\mathcal{R}}_n(\mathcal{F})&=\frac{1}{n}\mathbb{E}_{\epsilon}\sup_{f\in\mathcal{F}}\sum_{i=1}^n \epsilon_i\left(f(X_i)-\hat{f}_N(X_i)+\sum_{j=1}^N(\hat{f}_j(X_i)-\hat{f}_{j-1}(X_i))\right)\\
    & \leq \frac{1}{n}\mathbb{E}_{\epsilon}\sup_{f\in\mathcal{F}}\sum_{i=1}^n \epsilon_i\left(f(X_i)-\hat{f}_N(X_i)\right) +\sum_{i=1}^n\frac{1}{n}\mathbb{E}_{\epsilon}\sup_{f\in\mathcal{F}} \sum_{j=1}^N\epsilon_i\left(\hat{f}_j(X_i)-\hat{f}_{j-1}(X_i)\right)\\
    & \leq \|\epsilon\|_{L_2(P_n)} \sup_{f\in\mathcal{F}}\|f-\hat{f}_N\|_{L_2(P_n)}+\sum_{i=1}^n\frac{1}{n}\mathbb{E}_{\epsilon}\sup_{f\in\mathcal{F}} \sum_{j=1}^N\epsilon_i\left(\hat{f}_j(X_i)-\hat{f}_{j-1}(X_i)\right)\\
    &\leq \alpha_N+\sum_{i=1}^n\frac{1}{n}\mathbb{E}_{\epsilon}\sup_{f\in\mathcal{F}} \sum_{j=1}^N\epsilon_i\left(\hat{f}_j(X_i)-\hat{f}_{j-1}(X_i)\right),
\end{align*}
where $\epsilon=(\epsilon_1,\ldots,\epsilon_n)$ and the second-to-last inequality is due to Cauchy–Schwarz. Now the second term is the summation of empirical Rademacher complexity w.r.t. the function classes $\{f'-f'': f'\in T_j, f''\in T_{j-1}\}$, $j=1,\ldots, N$. Note that
\begin{align*}
    \|\hat{f}_j-\hat{f}_{j-1}\|^2_{L_2(P_n)}&\leq \left(\|\hat{f}_j-f\|_{L_2(P_n)}+\|f-\hat{f}_{j-1}\|_{L_2(P_n)}\right)^2\\
    &\leq (\alpha_j+\alpha_{j-1})^2\\
    &=3 \alpha_j^2.
\end{align*}
Massart's lemma \cite[Theorem 3.7]{mohri2018foundations} states that if for any finite function class $\mathcal{F}$, $\sup_{f\in\mathcal{F}}\|f\|_{L_2(P_n)}\leq M $, then we have
\begin{align*}
\hat{\mathcal{R}}_n(\mathcal{F})\leq\sqrt{\frac{2 M^2\log (|\mathcal{F}|)}{n}}.
\end{align*}
Applying Massart's lemma to the function classes $\{f'-f'': f'\in T_j, f''\in T_{j-1}\}$, $j=1,\ldots, N$, we get that for any $N$,
\begin{align*}
    \hat{\mathcal{R}}_n(\mathcal{F})&\leq \alpha_N+\sum_{j=1}^N 3\alpha_j\sqrt{\frac{2\log (|T_j|\cdot|T_{j-1}|)}{n}}\\
    &\leq \alpha_N+6\sum_{j=1}^N\alpha_j\sqrt{\frac{\log (|T_j|)}{n}}\\
    &\leq \alpha_N+12\sum_{j=1}^N(\alpha_j-\alpha_{j+1})\sqrt{\frac{\log \mathcal{N}(\alpha_j,\mathcal{F}, L_2(P_n))}{n}}\\
    &\leq \alpha_N+12\int_{\alpha_{N+1}}^{\alpha_0}\sqrt{\frac{\log \mathcal{N}(r,\mathcal{F}, L_2(P_n))}{n}}dr,
\end{align*}
where the third inequality is due to $2(\alpha_j-\alpha_{j+1})=\alpha_j$. Now for any small $\delta>0$ we can choose $N$ such that $\alpha_{N+1}\leq \delta <\alpha_N$. Hence,
\begin{align*}
    \hat{\mathcal{R}}_n(\mathcal{F})&\leq 2\delta +12\int_{\delta/2}^{M}\sqrt{\frac{\log \mathcal{N}(r,\mathcal{F}, L_2(P_n))}{n}}dr.
\end{align*}
Since $\delta>0$ is arbitrary, we can take $\inf$ w.r.t. $\delta$ to get
\begin{align*}
    \hat{\mathcal{R}}_n(\mathcal{F})&\leq\inf_{0<\delta<M}\left( 4\delta +12\int_{\delta}^{M}\sqrt{\frac{\log \mathcal{N}(r,\mathcal{F}, L_2(P_n))}{n}}dr\right).
\end{align*}
The result follows due to the fact that
\begin{align*}
    \mathcal{N}(r,\mathcal{F}, L_2(P_n))\leq \mathcal{N}(\epsilon,\mathcal{F},L_{\infty}(P_n))\leq \mathcal{N}(\epsilon,\mathcal{F},\|\cdot\|_{\infty}).
\end{align*}
\end{proof}

\section{Proof of Theorem \ref{cor1}} \label{ap:f}

\begin{proof}
Taking $W_1 L_1=\ceil*{\sqrt{n}}$, \citet[Theorem 4.3]{shen} gives rise to $\mathcal{E}_1= O(\sqrt{d} n^{-\frac{1}{d+1}})$.
The range of $g$ and $e$ covers the supports of $\mu$ and $\nu$, respectively, hence Theorem \ref{lma6} leads to $\mathcal{E}_2=0$. By Lemma \ref{lma8}, we have
\begin{align*}
\log \mathcal{N}(\epsilon,\mathcal{F}^1, \|\cdot\|_{\infty})\leq \left(\frac{8\sqrt{2(d+1)}M}{\epsilon}\right)^{d+1}\log \frac{16 B}{\epsilon}.
\end{align*}
Now following the same procedure as in Section \ref{ap:d} by taking $\delta=8\sqrt{2(d+1)}n^{-\frac{1}{d+1}}(\log n)^{\frac{1}{d+1}}$, we have
\begin{align*}
    \max\{\mathbb{E}\mathcal{E}_3,\mathbb{E}\mathcal{E}_4\} =O\left(\sqrt{d} n^{-\frac{1}{d+1}}(\log n)^{\frac{1}{d+1}}\right).
\end{align*}
At last, we consider all four error terms simultaneously.
 \begin{align*}
 \mathbb{E}d_{\mathcal{F}^1}(\hbnu,\hbmu)
    & \le \mathcal{E}_1+ \mathcal{E}_2 + \mathbb{E}\mathcal{E}_4+\mathbb{E}\mathcal{E}_3\\
    &= O(\sqrt{d} n^{-\frac{1}{d+1}}+0+\sqrt{d} n^{-\frac{1}{d+1}}(\log n)^{\frac{1}{d+1}})\\
    &= O(\sqrt{d} n^{-\frac{1}{d+1}}(\log n)^{\frac{1}{d+1}}).
\end{align*}

\end{proof}

\section{Proof of Theorem \ref{cor2}}

Following the same proof as Theorem \ref{lma6}, we have the following theorem.
\begin{theorem}
Suppose $\nu$ supported on $\mathbb{R}^k$ and $\mu$ supported on $\mathbb{R}^d$ are both absolutely continuous w.r.t. Lebesgue measure, and $z_i's$ and $x_i's$ are i.i.d. samples from $\nu$ and $\mu$, respectively for $1\leq i\leq n$. Then there exist generator and encoder neural network functions $g: \mathbb{R}^k\mapsto\mathbb{R}^d$ and $e: \mathbb{R}^d\mapsto\mathbb{R}^k$ such that $g$ and $e$ are inverse bijections of each other between $\{z_i: 1\le i\le n\}$ and $\{x_i: 1\le i\le n\}$. Moreover, such neural network functions $g$ and $e$ can be obtained by properly specifying $W_2^2 L_2= c_2 d n $ and $W_3^2 L_3= c_3 k n$ for some constant $12\le c_2, c_3\le 384$.
\end{theorem}

Since $\mu$ and $\nu$ are absolutely continuous by assumption, they are also absolutely continuous in any one dimension. Hence the proof reduces to the one-dimensional case.

\section{Additional Lemma}

 Denote $\mathcal{S}^d (z_0,\ldots,z_{N+1})$ as the set of all continuous piecewise linear functions $f:\Rbb\mapsto \Rbb^d $ which have breakpoints only at $z_0 < z_1 < \cdots< z_N < z_{N+1}$ and are constant on $(-\infty,z_0)$ and $(z_{N+1},\infty)$. The following lemma is a result in \cite{yang2021capacity}.

\begin{lemma}\label{appdis}
	Suppose that $W\ge 7d+1$, $L\ge2$ and $N\le (W-d-1)\floor*{\frac{W-d-1}{6d}}\floor*{\frac{L}{2}}$. Then for any $z_0<z_1<\cdots<z_N<z_{N+1}$, $\mathcal{S}^d (z_0,\ldots,z_{N+1})$ can be represented by a ReLU FNNs with width and depth no larger than $W$ and $L$, respectively.
\end{lemma}
This result indicates that the expressive capacity of ReLU FNNs for piecewise linear functions. If we choose $N=(W-d-1)\floor*{\frac{W-d-1}{6d}}\floor*{\frac{L}{2}}$, a simple calculation shows $c W^2 L/d\le N\le C W^2 L/d$ with $c=1/384$ and $C=1/12$. This means when the number of breakpoints are moderate compared with the network structure, such piecewise linear functions are expressible by feedforward ReLU networks.



\end{document}